\documentclass[sn-mathphys-ay]{sn-jnl}

\usepackage{graphicx}%
\usepackage{multirow}%
\usepackage{amsmath,amssymb,amsfonts}%
\usepackage{amsthm}%
\usepackage{mathrsfs}%
\usepackage[title]{appendix}%
\usepackage{xcolor}%
\usepackage{textcomp}%
\usepackage{manyfoot}%
\usepackage{booktabs}%
\usepackage{algorithm}%
\usepackage{algorithmicx}%
\usepackage{algpseudocode}%
\usepackage{listings}%

\newcommand{\blambda}{\mbox{\boldmath{${\lambda}$}}}
\newcommand{\bTau}{\mbox{\boldmath{${\tau}$}}}
\newcommand{\bSigma}{\mbox{\boldmath{${\sigma}$}}}

 \def\X{\mathbf{X}}

    \def\XX{\mathbf{X}_{\bigcdot \bigcdot}}
    \def\tilXX{\mathbf{\tilde{X}}_{\bigcdot \bigcdot}}
  \def\y{\mathbf{y}}

  \def\Sol{\{\S_{\bigcdot \bigcdot}^{(k)}\}_{k=1}^K}
    
\def\SolHat{\{\hat{\S}_{\bigcdot \bigcdot}^{(k)}\}_{k=1}^K}
\def\SolTil{\{\tilde{\S}_{\bigcdot \bigcdot}^{(k)}\}_{k=1}^K}
\def\fpen{f_{\text{pen}}}
\def\SolSpace{\mathbb{S}_{\hat{X}}}
\def\Sk{\S_{\bigcdot \bigcdot}^{(k)}}
\def\hatSk{\hat{S}_{\bigcdot \bigcdot}^{(k)}}
\def\tilSk{\tilde{S}_{\bigcdot \bigcdot}^{(k)}}

\def\y{\mathbf{y}}

\def\I{\mathbf{I}}
\def\U{\mathbf{U}}

\def\E{\mathbf{E}}
\def\EE{\mathbb{E}}
\def\V{\mathbf{V}}
\def\A{\mathbf{A}}

\def\R{\mathbf{R}}

\def\C{\mathbf{C}}

\def\D{\mathbf{D}}

\def\S{\mathbf{S}}
\def\0{\mathbf{0}}

\def\log{\mbox{log}}

\def\min{\mbox{min}}
\def\Sol{\{\S_{\bigcdot \bigcdot}^{(k)}\}_{k=1}^K}
\def\SolHat{\{\hat{\S}_{\bigcdot \bigcdot}^{(k)}\}_{k=1}^K}
\def\SolTil{\{\tilde{\S}_{\bigcdot \bigcdot}^{(k)}\}_{k=1}^K}
\def\fpen{f_{\text{pen}}}
\def\SolSpace{\mathbb{S}_{\hat{\X}}}
\def\Sk{\S_{\bigcdot \bigcdot}^{(k)}}
\def\hatSk{\hat{\S}_{\bigcdot \bigcdot}^{(k)}}
\def\tilSk{\tilde{\S}_{\bigcdot \bigcdot}^{(k)}}

  \makeatletter
\newcommand*\bigcdot{\mathpalette\bigcdot@{.5}}
\newcommand*\bigcdot@[2]{\mathbin{\vcenter{\hbox{\scalebox{#2}{$\m@th#1\bullet$}}}}}
\makeatother

\theoremstyle{thmstyleone}%
\newtheorem{theorem}{Theorem}
\newtheorem{proposition}[theorem]{Proposition}%
\newtheorem{corollary}[theorem]{Corollary}
\newtheorem{lemma}[theorem]{Lemma}

\theoremstyle{thmstyletwo}%

\theoremstyle{thmstylethree}%

\begin{document}

\title[Empirical Bayes Linked Matrix Decomposition]{Empirical Bayes Linked Matrix Decomposition}


\author*[1]{\fnm{Eric F.} \sur{Lock}}\email{elock@umn.edu}

\affil*[1]{\orgdiv{Division of Biostatistics and Health Data Science}, \orgname{School of Public Health, University of Minnesota}, \orgaddress{\city{Minneapolis}, \postcode{55455}, \state{MN}, \country{USA}}}


\abstract{Data for several applications in diverse fields can be represented as multiple matrices that are linked across rows or columns.  This is particularly common in molecular biomedical research, in which multiple molecular ``omics" technologies may capture different feature sets (e.g., corresponding to rows in a matrix) and/or different sample populations (corresponding to columns).  This has motivated a large body of work on integrative matrix factorization approaches that identify and decompose low-dimensional signal that is shared across multiple matrices or specific to a given matrix.  We propose an empirical variational Bayesian approach to this problem that has several advantages over existing techniques, including the flexibility to accommodate shared signal over any number of row or column sets (i.e., bidimensional integration), an intuitive model-based objective function that yields appropriate shrinkage for the inferred signals, and a relatively efficient estimation algorithm with no tuning parameters.  A general result establishes conditions for the uniqueness of the underlying decomposition for a broad family of methods that includes the proposed approach.  For scenarios with missing data, we describe an associated iterative imputation approach that is novel for the single-matrix context and a powerful approach for ``blockwise" imputation (in which an entire row or column is missing) in various linked matrix contexts.  Extensive simulations show that the method performs very well under different scenarios with respect to recovering underlying low-rank signal, accurately decomposing shared and specific signals, and accurately imputing missing data. The approach is applied to gene expression and miRNA data from breast cancer tissue and normal breast tissue, for which it gives an informative decomposition of variation and outperforms alternative strategies for missing data imputation. }




\maketitle

\section{Introduction}
\label{intro}

Low-rank matrix factorization techniques are widely used for many machine learning tasks such as data compression and dimension reduction, denoising to approximate underlying signal, and matrix completion.  Often, instead of a single matrix, the data for a given application takes the form of multiple matrices that are linked by rows and/or columns.  For example, in Section \ref{app} we consider data from The Cancer Genome Atlas, in which  gene expression (mRNA) and microRNA (miRNA) data are available from different platforms for breast cancer tumor samples and normal breast tissue samples from unrelated individuals.  Thus, the data can be represented as four matrices: (1) mRNA for cancer tissues, (2) mRNA for normal tissues, (3) miRNA for cancer tissues, and (4) miRNA for normal tissues; these matrices share both rows (here, molecular features) columns (here, samples), giving a {\it bidimensionally linked} structure. Furthermore, these data have {\it column-wise} missingness, in which molecular data for one of the platforms (mRNA or miRNA) is entirely missing for some samples in each of the cancer and control cohorts.  We are interested in imputing the missing data, as well as identifying patterns of variation that are shared or specific to the different data sources and sample cohorts; for these tasks, a low-rank factorization approach that integrates all of the matrices together while accounting for their differences is attractive.  

There is a growing and dynamic literature on multi-matrix low-rank factorization approaches.  The majority of existing methods assume that just one dimension, rows or columns, is shared across matrices, i.e., {\it unidimensional} integration. For example, the Joint and Individual Variation Explained (JIVE) approach \citep{lock2013joint,oconnell2016} decomposes low-rank structure that is shared across matrices from variation that is specific to each matrix.  Several related approaches yield a similar decomposition using different strategies or assumptions \citep{feng2018angle, schouteden2014performing,yang2015non,tang2021integrated}, while other approaches such as structural learning and integrative decomposition (SLIDE) \citep{gaynanova2019structural} allow for structure that is partially shared across any subset of matrices.  Such an approach is efficient and useful for interpretation, as for many applications identifying shared, partially shared, and specific underlying structure is of interest.  

The bidimensional linked matrix factorization (BIDIFAC) approach \citep{park2020integrative,lock2020bidimensional} was developed to accommodate bidimensionally linked integration scenarios, by identifying structure that is shared or specific across both row sets and column sets.  The objective function for BIDIFAC involves a structured nuclear norm penalty, which approximates the data as a decomposition of low-rank {\it modules} that are shared across row or column subsets.  This approach is attractive in the bidimensionally linked context because it has favorable convergence properties and generally leads to a uniquely-defined decomposition \citep{lock2020bidimensional}; methods for unidimensional integration without shrinkage penalties typically require orthogonality \citep{lock2013joint, feng2018angle, gaynanova2019structural,yi2023hierarchical} or other constraints to assure uniqueness, which do not extend to the bidimensionally linked context.  Furthermore, the application of the structured nuclear norm penalty for unidimensional integration (termed  UNIFAC) has been shown to outperform alternative approaches with respect to accurately decomposing shared and specific structures \citep{lock2020bidimensional} and imputing missing data \citep{park2020integrative} under certain scenarios. However, a drawback is that it tends to over-penalize the estimated structure, and is thus biased toward an over-conservative solution with respect to the underlying low-rank structure \citep{yi2023hierarchical}.

For a single matrix the tendency of the nuclear norm to over-shrink low-rank signal has been well-documented, and in contrast low-rank approximations without any further shrinkage tend to overfit \citep{josse2016adaptive}.  Empirical Bayes approaches to low-rank factorization are an attractive compromise, in which the optimal levels of shrinkage are determined empirically within a model-based framework \citep{wang2021empirical}. Under a variational Bayes approach, the true posterior distribution for Bayesian model is approximated by minimizing the free energy (reducing to Kullback-Leibler divergence) under simpler distributional assumptions \citep{fox2012tutorial}.  In particular, an empirical variational Bayes approach to low-rank matrix factorization under an assumption of normally distributed factor matrices has desirable theoretical properties, and the resulting approximation for a scenario without missing data can be efficiently computed via a closed-form shrinkage formula on the observed singular values of the matrix \citep{nakajima2012perfect,nakajima2013global}.   

In this article, we describe an empirical variational Bayesian approach to identify and decompose shared, partially shared, or specific low-rank structure in the context of unidimensionally or bidimensionally linked matrices.  Moreover, we describe an accompanying iterative imputation approach analogous to an expectation-maximization (EM) algorithm that can handle entry-wise missing or row/column-wise missing values.  The proposed approach is free of any tuning parameters, and is shown to share desirable properties of the BIDIFAC approach while addressing the issue of over-shrinkage.

\section{Notation and Setting}
\label{notation}

Throughout this article bold uppercase characters (`$\X$') correspond to matrices and bold lowercase characters (`$\mathbf{x}$') correspond to vectors. Square brackets index the entries within a vector or matrix, e.g., `$\X[m,n]$'.  The singular value decomposition (SVD) of a matrix $\X: M \times N$ is given by $\U_\X \D_\X \V_\X^T$, where the diagonal entries of $\D_\X$, $\D_\X[r,r]$,  are the singular values of $\X$ ordered from largest to smallest for $r=1,\hdots, \min(M,N)$.  The Frobenius norm is given by $||\cdot||_F$, so that $||X||_F^2$ is the sum of the squared entries in $\X$; the nuclear norm is given by $||\cdot||_*$, which is the sum of the singular values in a matrix: \[||\X||_*=\sum_{r=1}^{\min(M,N)} \D_\X[r,r].\]

  We consider the setting of a collection of matrices  $\X_{ij}: M_i \times N_j$ for $i=1,\hdots,I$ and $j=1,\hdots,J$. The column sets $N_1,\hdots, N_J$ and row sets $M_1,\hdots,M_L$ are consistent across the matrices.   The subscript `$\bigcdot$' defines concatenation across row or column sets, and the full data $\XX$ can be represented as a single $M \times N$  matrix where $M = M_1+\hdots+M_I$ and $N = N_1+\hdots+N_J$:   
\begin{align}\XX = \left [ \begin{array}{c c c c} \X_{11} & \X_{12} & \hdots & \X_{1J}  \\  \vdots & \vdots & \vdots & \vdots \\ \X_{I1} & \X_{I2} & \hdots & \X_{IJ} \end{array} \right ] \; \text{  where  } \X_{ij} \text{  are  } M_i \times N_j. \label{bimat}\end{align}
As for our motivating application in Section~\ref{app}, the $I$ column sets may correspond to different sample cohorts and the $J$ row sets may correspond to features from different high-dimensional sources.   

\section{Single-matrix results}
\label{mat.prelim}
Here we review some established results for a single matrix that are critical to what follows. Propositions \ref{prop:1} and \ref{prop:2} describe solutions to the least squares matrix approximation problem under a rank-constrained or nuclear norm penalized criterion, respectively.

\begin{proposition} \label{prop:1} For $\X: M \times N$, the minimizer of the least squares objective $||X-\S||_F^2$ under the constraint rank$(\S)=R \leq \min (M,N)$ is given by $\S=\U_\X \D_S \V_\X^T$ where $\D_\S$ is diagonal with $\D_\S[r,r]=\D_\X[r,r]$ for $r=1,\hdots,R$ and $\D_\S[r,r]=0$ for $r > R$.  
\end{proposition}

\begin{proposition} \label{prop:2} For $\X: M \times N$, the minimizer of the nuclear norm penalized objective $\frac{1}{2}||X-\S||_F^2+\lambda ||\S||_*$ is given by $\S=\U_\X \D_S \V_\X^T$ where $\D_S$ is diagonal with $\D_\S[r,r]=\max(\D_\X[r,r]-\lambda,0)$ for all $r$. 
\end{proposition}
Proposition \ref{prop:1} is well-known, and a proof of Proposition~\ref{prop:2} can be found in \citet{mazumder2010spectral}.  Because of how they operate on the singular values $\X$, Proposition \ref{prop:1} represents a {\it hard-thresholding} and Proposition \ref{prop:2} a {\it soft-thresholding} approach to low-rank approximation.  Proposition~\ref{prop:3} below establishes that the soft-thresholding operator also solves an objective with $L_2$-penalties on the factor components of $\S$, while Corollary~\ref{cor:1} further establishes that it is the posterior mode of a Bayesian model with normal priors on the factor components.

 \begin{proposition} \label{prop:3} For $\X: M \times N$,
\begin{align}
    \underset{\U: M\times H,\V: N\times H}{\min} ||\X-\U \V^T||_F^2+\lambda (||\U||_F^2+||\V||_F^2) = \underset{\S: \text{rank}(\S)\leq H}{\min} ||\X-\S||_F^2 +2\lambda ||\S||_* \label{nucpen}
\end{align}
   Moreover, $\widehat{\S}=\widehat{\U}\widehat{\V}^T$, where $\widehat{\S}$ solves the right-hand side of \eqref{nucpen} and $\{\widehat{\U},  \widehat{\V}\}$ solves the left-hand side of \eqref{nucpen}.
\end{proposition}

 \begin{corollary} \label{cor:1} For $\X: M \times N$,
if $\U: M \times H$ and $\V: N \times H$ have independent Normal$(0, \sigma^2/\lambda)$ entries, and $\X=\U \V^T+\E$ where $\E$ has independent Normal$(0,\sigma^2)$ entries, then the mode of the posterior distribution $p(\U,\V \mid \X)$ satisfies $\widehat{\S}=\widehat{\U}\widehat{\V}^T$, where $\widehat{\S}$ solves the right-hand side of \eqref{nucpen}.
\end{corollary}
A proof of Proposition~\ref{prop:3} can be found in \citet{mazumder2010spectral}, and Corollary~\ref{cor:1} follows by noting that the log of the posterior density for the specified model is proportional to the left-hand side of \eqref{nucpen}.  While the results hold for general $H$, in what follows we set $H$ as the largest possible rank $H= \min (M,N)$; the actual estimated rank will tend to be smaller due to singular value thresholding (as in Proposition~\ref{prop:2}).     The distribution of singular values for the error matrix $\E$ has been well-studied, including the following results on the first singular value under general conditions from \citet{rudelson2010non}:
  \begin{proposition} \label{prop:4} If $\E: M \times N$ has independent entries with mean $0$, variance $\sigma^2$ and finite fourth moment, then $\D[1,1] \approx \sigma (\sqrt(M)+\sqrt{N}$) as $M,N \rightarrow \infty$; further, if the entries are Gaussian, $\mathbb{E}(\D[1,1]) \leq \sigma (\sqrt(M)+\sqrt{N})$ for any $M,N$.
  \end{proposition}  

These results can motivate choice of the nuclear norm penalty $\lambda$.  A conservative approach, if the error variance $\sigma^2$ is known, is to select $\lambda = \sigma (\sqrt{M}+\sqrt{N})$; per Propositions~\ref{prop:4} and \ref{prop:2}, this will tend to shrink to $0$ any singular values in $\X$ that result from error with no true low-rank signal.  Another approach is to adopt an empirical Bayes estimate for $\lambda$ based on the model in Corollary~\ref{cor:1}. However, a general limitation of the nuclear norm approach to low-rank matrix reconstruction is the use of a universal threshold on all singular values; in practice larger singular values should be penalized less, as they capture a relatively lower ratio of error and higher signal  \citep{shabalin2013reconstruction}.  An analogous interpretation is that the entries of $\U$ and $\V$ should not have identical variance across columns, and this can be relaxed by a model in which they have column-specific variances: 
$\U[\cdot,r] \sim \text{Normal}(\0,\tau^2_{\U,r} \I)$ and $\V[\cdot,r] \sim \text{Normal}(\0,\tau^2_{\V,r} \I)$ for $r=1,\dots,H$.\footnote{For most scenarios $\tau^2_{\U,r}$ and $\tau^2_{\V,r}$ are not independently identifiable, as the level of signal is determined by their product.  } 

For given $\bTau=\{\tau_{\U,r} \tau_{\V,r}\}_{r=1}^H$, the approximation resulting from the posterior mode for $p(\U,\V \mid \X, \bTau)$ is given by thresholding the singular values of $\X$, analogous to Corollary~\ref{cor:1}.  To determine $\bTau$, directly maximizing the joint likelihood $p(\X, \U, \V \mid \bTau)$ over $\U, \V$ and $\bTau$ degenerates to $\bTau \rightarrow \0$ in the global solution, though it can have a reasonable local mode that is a singular value thresholding operator \citep{nakajima2014analysis}.  Further, a marginal empirical Bayes estimator that optimizes $p(\X \mid \bTau)$ (effectively integrating over $\U$ and $\V$) is difficult to obtain computationally or analytically. As an alternative, one can use an empirical variational Bayes (EVB) approach.  In general, for a variational Bayes approach \citep{attias1999}, the true posterior of a Bayesian model with parameters $\Theta$ and data $\y$, $p(\Theta \mid \y)$ is approximated by another distribution $q(\Theta)$ under simplifying assumptions to minimize the free energy
\[F(q) = E_q \log \frac{ q(\Theta)}{p(\Theta,\y)} = E_q \log \frac{ q(\Theta)}{p(\Theta \mid \y)}- \log \; p(\y),\]
where $E_q(\cdot)$ is the expected value with respect to the distribution defined by $q(\cdot)$. 
In the context of low-rank approximation, a useful simplifying assumption for $q(\cdot)$ is that the left and right factor matrices are independent, $q(\U,\V)=q_u(\U) q_v(\V)$.  Thus, for hyperparameters $\bTau$ and $\sigma$ the free energy to minimize is
\begin{align}F(q \mid \sigma, \bTau) = E_q \log \frac{ q_u(\U)q_v(\V)}{p(\X \mid \U, \V, \sigma) p(\U \mid \bTau_\U) p(\V \mid \bTau_\V)}, \label{freeenergy}\end{align}
and a point estimate for low-rank structure $\hat{\S}$ can be obtained via its expected value under $q(\cdot)$: 
\begin{align} \hat{\S}=E_q (\U \V^T). \label{freeapprox}\end{align}
Under an empirical variational Bayes approach, $\bTau$ and $\sigma$ are also estimated by minimizing the free energy in \eqref{freeenergy}.  For fully-observed $\X$, \citet{nakajima2012perfect} showed that the global solution to the empirical variation Bayes objective can be obtained in closed form as a singular value thresholding operator; their result for fixed $\sigma$ is given in Theorem~\ref{thm:1}, and for estimation of $\sigma$ in Theorem~\ref{thm:2}.
 \begin{theorem}
 \label{thm:1} The approximation $\hat{\S}$ \eqref{freeapprox} obtained by minimizing \eqref{freeenergy} over $q$ and $\bTau$ is $\hat{\S}=\S=\U_\X \D_S \V_\X^T$ where $\D_S$ is diagonal with 
 \[\D_\S[r,r]=0 \text{ if } D_\X[r,r]<\sigma \sqrt{M+N+\sqrt{MN} (\kappa + 1/\kappa)}\] and 
 \[\D_\S[r,r]= \left(\D_\X[r,r]^2-(M+N) \sigma^2+\sqrt{(\D_\X[r,r]^2-(M+N)\sigma^2)^2-4MN\sigma^4} \right)/(2\D_\X[r,r]) \text{ otherwise.}\]
The value $\kappa$ is given as the zero-cross point of the function
 \[f(x)=x \sqrt{N/M}\,  log \left( x\sqrt{M/N}+1 \right) +x\sqrt{M/N} \, log \left( x\sqrt{N/M}+1 \right) -1. \]
 \end{theorem}
 
  \begin{theorem}
 \label{thm:2} Assume without loss of generality that $M>N$ and $\alpha=N/M$. The estimate $\hat{\sigma}$ that minimizes \eqref{freeenergy} is the global minimizer of  
 \[\Psi(\sigma)=\sum_{r=1}^N \psi_1 \left(\frac{\D_X[r,r]^2}{M \sigma^2}  \right),\]
 where \[\psi_2(x) = x-\log \,x+{1}_{\{x>c\}} \psi_2(x),\]
 \[\psi_2(x) = \log \, (\sqrt{\alpha} \psi_3(x)+1)+\alpha \log \, (\psi_3(x)/\sqrt{\alpha}+1)-\alpha \psi_3(x),\]
 and 
 \[\psi_3(x)=\frac{1}{2 \sqrt{\alpha}} \left(x-1-\alpha + \sqrt{(x-(1+\alpha))^2-4 \alpha}\right).\]
\end{theorem}

Together, Theorems~\ref{thm:2} and \ref{thm:1} provide a powerful and efficient way to approximate low-rank signal without any tuning parameters.  The only substantial computing requirement is in obtaining the SVD of $\X$, as other steps involve univariate functions that are trivial to solve.

\section{Bidimensional linked matrix factorization}

BIDIFAC+ \citep{lock2020bidimensional} was developed to decompose bidimensionally linked data as in~\eqref{bimat} into modules of low-rank structure that are shared across column sets and/or row sets.      That is,
\begin{align} \label{bidiflex} \X_{\bigcdot \bigcdot} \approx \sum_{k=1}^K \Sk,  \end{align}  
  where each $\Sk$ is a concatenation of blocks $\S^{(k)}_{ij}$ as in \eqref{bimat} and 
      $$\S^{(k)}_{ij} = \begin{cases}
	\0_{M_i \times N_j}  & \text{if } \R[i,k] = 0 \text{ or } \C[j,k] = 0\\
 	\U_{i}^{(k)} \V_{j}^{(k) T} & \text{if } \R[i,k] = 1 \text { and  }\C[j,k] = 1 \end{cases}. $$
  Here, $\R$ and $\C$ are binary matrices that indicate the presence of module $k$ ($\Sk$) across the row and column sets; that is, $\Sk$ is low-rank structure specific to the submatrix  defined by the row sets identified in $\R[\bigcdot,k]$ and the column sets identified in $\C[\bigcdot,k]$.  This general framework subsumes several other integrative decompositions of structured variation. For example, in the unidimensionally linked case with shared columns $J=1$, JIVE \citep{lock2013joint} identifies components that are either shared across all matrices ($\R[i,k]=1$ for all $i$) or specific to a single matrix ($\R[i,k]=1$ for just one $i$), while SLIDE \citep{gaynanova2019structural} allows for components that are shared across any number of matrices.  The BIDIFAC  method \citep{park2020integrative} is a special case in which modules are either globally shared ($\R[i,k]=\C[j,k]=1$ for all $i,j$), shared across all rows for a given column set ($\R[i,k]=1$ for all $i$ and $\C[j,k]=1$ for one $j$), shared across all columns for a given row set ($\C[j,k]=1$ for all $j$ and $\R[i,k]=1$ for one $i$), or specific to one matrix ($\R[i,k]=1$ and $\C[j,k]=1$ for just one $i,j$). 
  
 The BIDIFAC and BIDIFAC+ approaches solve a structured nuclear norm objective: 
\begin{align} \underset{\Sol}{\mbox{argmin}} \; \; \frac{1}{2}||\X_{\bigcdot \bigcdot}-\sum_{k=1}^K \Sk ||^2_F + \sum_{k=1}^K \lambda_k ||\Sk||_*,  \label{obj_eq} \end{align}
subject to $\S^{(k)}_{ij} = \0_{M_i \times N_j}$ if  $\R[i,k] = 0$ or $\C[j,k] = 0$.  For fixed $\R$ and $\C$, objective (\ref{obj_eq}) may be solved with an iterative soft-singular value thresholding algorithm that cycles through the estimation of each $\Sk$  with applications of Proposition~\ref{prop:3}. For smaller $I$ and $J$, $\R$ and $\C$ can enumerate all possible submatrices of the row and column sets; the solution may still be sparse, with $\Sk=\0$ if the submatrix corresponding to module $k$ has no particular shared structure.  Alternatively, if the number of row and column sets $I$ and $J$ are large, \citet{lock2020bidimensional} describe an algorithm to adaptively estimate $\R$ and $\C$ to give the submatrices with the most significant low-rank structure during estimation.

The choice of the module-specific penalty parameters $\lambda_k$ is critical.  By default, $\lambda_k$ are selected by an approach motivated by Proposition~\ref{prop:4}; under the assumption that each matrix $\X_{ij}$ has error variance $1$, $\lambda_k$ depends on the total number of non-zero rows and columns in its submatrix,
\[\lambda_k = \sqrt{\sum_{i=1}^I \R[i,k] M_i}+ \sqrt{\sum_{j=1}^J \C[j,k] N_j},\]
for $k=1,\hdots,K$.  This approach will shrink to $0$ any singular values in the submatrix that are due to error only.  In practice each data matrix $\X_{ij}$ is scaled to have error variance $1$ beforehand; if the error variance is unknown, then it can be taken as the total variance in $\X_{ij}$ (a conservative choice) or via other approaches such as the median absolute deviation estimator in \citet{gavish2017optimal}.


\section{Empirical variational BIDIFAC (EV-BIDIFAC)}
\label{EV-BIDIFAC}

While the BIDIFAC+ approach has several advantages, one weakness is the required standardization by the noise variance in each data matrix, which is estimated pre-hoc and not within a cohesive framework.  More critically, the use of nuclear norm penalties generally results in over-shrinkage of the signal, and the default theoretically-motivated choice of the $\lambda_k$'s is particularly conservative.  Thus, we aim to develop a more flexible empirical variational Bayes approach to the decomposition of bidimensionally linked matrices, analogous to that for  a single matrix in Section~\ref{mat.prelim}.  Our psuedo-Bayesian model is as follows for $i=1,\hdots I$ and $j=1,\hdots,J$: 
\begin{align}\X_{ij} = \sum_{k=1}^K \U_{i}^{(k)} \V_{j}^{(k) T} + \E_{i,j} \text{ where } \E_{i,j}[l,m] \overset{iid}{\sim} \text{Normal}(0, \sigma^2_{ij}), \label{bayes_mod} \\
\U_i^{(k)}[\cdot,r]=\0 \text{ if } \R[i,k]=0\text{ and  }\U_i^{(k)}[\cdot,r]\sim \text{Normal}(\0, \sigma_{ij} \bTau^2_u[k,r] \I) \text{ if } \R[i,k]=1, \notag  \\   
\V_j^{(k)}[\cdot,r]=\0 \text{ if } \C[j,k]=0\text{ and  }\V_j^{(k)}[\cdot,r] \sim \text{Normal}(\0, \sigma_{ij} \bTau^2_v[k,r] \I ) \text{ if } \C[j,k]=1, \notag  \end{align}
where $\I$ is an identity matrix of appropriate dimension. 
 Using repeated applications of Corollary~\ref{cor:1}, the BIDIFAC+ objective can be shown to give the posterior mode of the model specified in \eqref{bayes_mod} for which $\bSigma=\{\sigma_{ij}: i=1,\hdots,I, j=1,\hdots,J\}$ and $\bTau=\{\bTau_u, \bTau_v\}$ are fixed, and $\bTau^2_u[k,r]=\bTau^2_v[k,r]=1/\lambda_k$ for all $r$.  We extend this framework by allowing $\bSigma$ and $\bTau$ to be estimated within a cohesive empirical variational Bayes framework.  We approximate the true posterior with a distribution $q(\cdot)$, for which the factor matrices are mutually independent: 
 \[q(\U,\V)=\prod_{k=1}^K q_{u,k}(\U^{(k)}) q_{v,k}(\V^{(k)}).\]
The corresponding free energy is given by 
\begin{align}F(q \mid \bSigma, \bTau) = E_q \log \frac{\prod_{k=1}^K q_{u,k}(\U^{(k)}) q_{v,k}(\V^{(k)})}{p(\XX \mid \U, \V, \bSigma) p(\{\U^{(k)}\}_{k=1}^K \mid \bTau_\U) p(\{\V^{(k)}\}_{k=1}^K \mid \bTau_\V)}, \label{freeenergy_bidi}\end{align}
with the densities $p(\cdot)$ in the denominator as specified in model~\eqref{bayes_mod}.  The resulting low-rank fit and decomposition is then given by 
\begin{align} \hat{\S}_{\bigcdot \bigcdot}= E_q \left( \sum_{k=1}^K \U_{\cdot}^{(k)} \V_{\cdot}^{(k) T} \right) = \sum_{k=1}^K \hat{\S}_{\bigcdot \bigcdot}^{(k)} \label{bidi_est} \end{align} 
where for $k=1,\hdots,K$,
\[\hat{\S}_{\bigcdot \bigcdot}^{(k)} = E_{q_u,q_v} \left( \U_{\cdot}^{(k)} \V_{\cdot}^{(k) T} \right). \]

Algorithm~\ref{alg1} describes an iterative approach to compute the empirical variational Bayes solution.  This algorithm converges to a blockwise minimum of the free energy~\eqref{freeenergy_bidi} over $q$ and $\bTau$ with $\bSigma$ fixed at the distinct minimizer of \eqref{freeenergy_bidi} for each $\X_{ij}$ separately.  Alternatively, one can update the variances $\bSigma$ over the iterative algorithm to minimize free energy under the joint model.  In Algorithm~\ref{alg1}, the number of modules $K$ and the row- and column-set inclusions for the modules, $\R$ and $\C$, are assumed to be fixed.  For a smaller number of linked matrices, one can set $\R$ and $\C$ to enumerate all of the possible submatrices of shared, partially shared, or specific structure: K=$(2^I-1)(2^J-1)$.  The estimated decomposition may still be sparse in that $\hatSk=\0$ if module $k$ has no distinct low-rank structure, as demonstrated in Section~\ref{bidim.sim}.  If enumerating all submatrices of the row- and column-sets is not feasible due to a large number matrices, one can specify $K$, $\R$ and $\C$ based on prior knowledge. Alternatively,  we may approximate the fully enumerated decomposition with a large upper bound on the number distinct modules, $K$, and the row- and column-set inclusions $\R$ and $\C$ estimated empirically as in \citet{lock2020bidimensional}.                
   
  \begin{algorithm}
  \caption{Estimation of EV-BIDIFAC model for fully observed data} \label{alg1}
  \begin{enumerate}
  \item Estimate $\sigma_{ij}$ for each $\X_{ij}$, as in Theorem~\ref{thm:2}.
  \item Define $\tilXX$ as in \eqref{bimat}, with $\tilde{\X}_{i j}=\X_{ij}/\sigma_{ij}$ for $i=1,\hdots,I$ and $j=1,\hdots,J$.
  \item Initialize $\tilSk=\mathbf{0}$ for $k=1,\hdots,K$
  \item For $k=1,\hdots,K$:
  \begin{enumerate}
  \item Let $\tilXX^{(k)}$ be the submatrix of $\tilXX-\sum_{k'\neq k} \tilde{\S}_{\bigcdot \bigcdot}^{(k')}$ over row sets indicated by $\R[\cdot,k]$ and column sets $\C[\cdot,k]$.
  \item Update the non-zero submatrix in $\tilSk$ by the application of the singular value thresholding in Theorem~\ref{thm:1} to $\tilXX^{(k)}$.  
  \end{enumerate} 
  \item Repeat Step 3. until convergence of the $\{\tilSk\}_{k=1}^K$.
  \item Unscale the estimated structure: $\hat{\S}^{(k)}_{ij} = \sigma_{ij} \tilde{\S}^{(k)}_{ij}$ for $k=1,\hdots,K$, $j=1,\hdots,J$, $i=1,\hdots,I$.   
  \end{enumerate}	
  \end{algorithm}

\section{Uniqueness}
\label{uniqueness}
  
  The uniqueness of decompositions into shared, partially shared, and specific low-rank structure requires careful consideration, as the terms in the model are generally not uniquely defined in general without further constraints.  For unidimensional decompositions of low-rank structure, conditions of orthogonality are commonly used to assure the different components are uniquely identified \citep{lock2013joint,feng2018angle,gaynanova2019structural,yi2023hierarchical}; however, such conditions are not straightforward to extend to decompositions for bidimensionally linked data.  Alternatively, \citet{lock2020bidimensional} showed that the terms in the BIDIFAC+ solution are uniquely defined under certain conditions, without orthogonality, as the minimizer of the structured nuclear norm objective in \eqref{obj_eq}.     Here we extend this result to a much broader class of penalties, that includes the EV-BIDIFAC decomposition.   
  
   For fixed row-set and column-set indicators for the modules, $\R$ and $\C$, let $\SolSpace$ be the set of possible decompositions that yield a given approximation $\hat{\X}_{\bigcdot \bigcdot}$:
$$\SolSpace = \left \{ \Sol \mid \hat{\X}_{\bigcdot \bigcdot} = \sum_{k=1}^K \Sk \right \}.$$  
In most setting the cardinality of $|\SolSpace|$ is infinite without further conditions.  Define $P_{\blambda} (\cdot)$ for a vector $\blambda$ as a weighted sum of the singular values of a matrix:
\[P_{\blambda} (\S) = \sum \blambda[r] \D_{\S}[r,r],\]
and let $\fpen (\cdot)$ give a structured penalty on the different modules of a decomposition,:
\[\fpen (\Sol) = \sum_{k=1}^K  P_{\blambda_k} (\Sk),\] 
with non-zero entries for each $\blambda_k$.  Theorem~\ref{thm:ident} gives sufficient conditions for the minimizer of this penalty to be uniquely determined.   

  \begin{theorem}
  \label{thm:ident}
Consider $\SolHat \in \SolSpace$ and let $\U_{\bigcdot}^{(k)} \hat{\D}^{(k)} \V_{\bigcdot}^{(k) T}$ give the SVD of $\hatSk$ for $k=1,\hdots,K$. The following three properties uniquely identify $\SolHat$.
\begin{enumerate}
\item $\SolHat$ minimizes $\fpen (\cdot)$ over $\SolSpace$,   
\item $\{\hat{\U}_i^{(k)}[\bigcdot,r]: \R[i,k] = 1 \text{ and } \hat{\D}^{(k)}[r,r]>0\}$ are linearly independent for $i=1,\hdots I$,  
\item $\{\hat{\V}_j^{(k)}[\bigcdot,r]: \C[j,k] = 1 \text{ and } \hat{\D}^{(k)}[r,r]>0\}$ are linearly independent for $j=1,\hdots,J$.
\end{enumerate}
 \end{theorem}
 
 The proof of Theorem~\ref{thm:ident} is given in Appendix~\ref{p1}.  Note that the linear independence conditions are straightforward to verify, and are likely to hold in practice if the ranks of the terms in $\SolHat$ are small relative  to the matrix dimensions.  As noted in Proposition~\ref{prop5}, the terms in any accumulation point of Algorithm~\ref{alg1} also give a coordinate-wise optimum of $\fpen (\cdot)$ for a certain set of penalties $\{\blambda_k\}_{k=1}^K$, and thus the EV-BIDIFAC solution is unique in this sense.     

  \begin{proposition}
  \label{prop5}
If $\SolHat \in \SolSpace$ is an accumulation point of Algorithm~\ref{alg1}, then it is a coordinate-wise optimum of $\fpen (\cdot)$ over $\SolSpace$ for some $\{\blambda_k\}_{k=1}^K$. 
 \end{proposition}

\section{Missing data imputation}

Now we consider the context in which the data are not fully observed and missing data must be imputed.  Let $\mathcal{M}$ index observations in the full dataset $\XX$ that are not observed: $\mathcal{M}=\{(m,n): \XX[m,n] \text{ is missing}\}$.  We develop an Expectation-Maximization (EM) approach for imputation that is analogous to previous approaches described for hard-thresholding \citep{kurucz2007methods}, soft-thresholding \citep{mazumder2010spectral}, and BIDIFAC+ \citep{lock2020bidimensional}.  The procedure for fixed estimates of the noise variance for each matrix, $\sigma_{ij}$, is described in Algorithm~\ref{alg2}.  
    \begin{algorithm}
  \caption{Estimation of EV-BIDIFAC model for missing data} \label{alg2}
  \begin{enumerate}
  \item Initialize $\tilSk=\mathbf{0}$ for $k=1,\hdots,K$
  \item Define $\tilXX$ as in \eqref{bimat}, with $\tilde{\X}_{i j}=\X_{ij}/\sigma_{ij}$ for $i=1,\hdots,I$ and $j=1,\hdots,J$.
  \item Set $\tilXX[m,n]=\tilde{\S}_{\bigcdot \bigcdot}[m,n]$ for $(m,n) \in \mathcal{M}$, where $\tilde{\S}_{\bigcdot \bigcdot} = \sum_{k=1}^K \tilSk$.  
  \item Update $\{\tilSk\}_{k=1}^K$ as in Step 4. of Algorithm~\ref{alg1}.
  \item Repeat Steps 3. and 4. until convergence of the $\{\tilSk\}_{k=1}^K$.
  \item Unscale the estimated structure: $\hat{\S}^{(k)}_{ij} = \sigma_{ij} \tilde{\S}^{(k)}_{ij}$ for $k=1,\hdots,K$, $j=1,\hdots,J$, $i=1,\hdots,I$. 
  \item Impute $\XX[m,n]= \hat{\S}_{\bigcdot \bigcdot}[m,n]$ for $(m,n) \in \mathcal{M}$.  
  \end{enumerate}	
  \end{algorithm} 

Note that Step 4.\ in Algorithm~\ref{alg2} is a partial maximization step,  as the free energy in~\eqref{freeenergy_bidi} is conditionally minimized for each module given the current imputed values, while Step 5.\ is an expectation step replacing the missing data with their expected value \eqref{bidi_est} under the current model fit. Further, the algorithm can be shown to minimize the free energy over the missing data $\{\XX[m,n] : (m,n) \in \mathcal{M}\}$; this result is given as Theorem~\ref{thm:3}, and is appealing because all unknowns (model hyperpameters, parameters and missing entries) are then estimated via optimizing the same unified free energy objective.  

 \begin{theorem}
 \label{thm:3}
Algorithm~\ref{alg2} iteratively minimizes the free energy \eqref{bidi_est} over $q(\cdot)$, $\bTau$, and $\{\XX[m,n] : (m,n) \in \mathcal{M}\}$.    
 \end{theorem} 
 
 A proof of Theorem~\ref{thm:3} is provided in Appendix~\ref{p2}. For the choice of $\bSigma$, we suggest two different approaches, depending on the type of missingness.  If entries are missing from an entire row or columns of a given data matrix $\X_{ij}$ (i.e., row- or column-wise missing), then $\sigma_{ij}$ can be estimated by the applying Theorem 2 to the complete submatrix $\{\XX[m,n] : (m,n) \notin \mathcal{M}\}$.  If the matrix is missing entries without entire rows or columns, then we update $\sigma_{ij}$ iteratively during Algorithm~\ref{alg2} as follows at the end of each cycle: 1. apply Theorem 2 to the matrix $\X_{ij}$ with currently imputed values to obtain $\sigma'_{ij}$, 2. update $\sigma_{ij}=\sigma'_{ij} \left(\frac{M_iN_j}{M_iN_j-|\mathcal{M}_{ij}|} \right)$ where $|\mathcal{M}_{ij}|$ is the number of missing entries in $\X_{ij}$.




\section{Simulations}

Here we present a collection of simulations to illustrate and compare the proposed approach with respect to various criteria.  We explore several conditions covering different signal-to-noise ratios, missing data scenarios, and linked matrix patterns with shared and specific structure.  Further simulation results (e.g., exploring different matrix dimensions) are available at\\
\url{https://ericfrazerlock.com/ev-bidifac-sims.html}.     

\subsection{Single matrix}
\label{single_sim}

We first describe a simple simulation on a single matrix to illustrate the empirical variational Bayes approach to matrix factorization as a flexible compromise between a hard-thresholding and a conservative soft-thresholding approach.  We consider a fully observed matrix $\X: 1000 \times 100$, with $\X= c \U \V^T + \E$ where the entries of $\E$, $\U: M \times 10$ and $\V: N \times 10$ are each independent standard normal, and $c$ is a manipulated value controlling the signal-to-noise ratio.  We consider $10$ values of $c$ distributed on a log-uniform scale between $c=0.05$ and $c=1$, and generate $10$ replications of the simulation for each value of $c$. For each generated dataset, we apply the following approaches:
\begin{itemize}
\item OPT: The oracle operator on the singular values when the structure $\S=c \U \V^T$ is known: $\hat{\S}_{\text{OPT}}=\U_X \hat{\D} \V_X^T$ where $\hat{\D}$ is diagonal with entries that solve the OLS problem to minimize $||\S-\U_X \hat{\D} \V_X^T||_F^2$. 
\item EVB: the empirical variational Bayes approach applying Theorems~\ref{thm:1} and \ref{thm:2}.
\item HT: The hard-thresholding approach described in Proposition~\ref{prop:1}, using the true rank $R=10$.
\item NN: The nuclear norm penalized soft-thresholding approach described in Proposition~\ref{prop:2}, with $\lambda=\sqrt{M}+\sqrt{N}$ (this correspond to a threshold motivated by Proposition~\ref{prop:4} with true noise variance $\sigma=1$).      
\end{itemize}
For each setting and for each method we compute the relative squared error (RSE) and oracle-normalized squared error (ONSE)
 \begin{align}\text{RSE}=\frac{||\S-\hat{\S}||_F^2}{||\S||_F^2} \text{ and } \text{ONSE}=\frac{||\S-\hat{\S}||_F^2}{||\S-\hat{\S}_{\text{OPT}}||_F^2}. \label{RSE_ONSE} \end{align}
Figure~\ref{fig1} shows the mean values for both RSE and ONSE under the different methods as a function of the signal-to-noise $c$. The hard-thresholding approach tends to overfit the signal, which can lead to particularly poor recovery when the signal is smaller relative to the noise. In contrast, the nuclear norm penalty is appropriately conservative in that it never overfits and performs better for lower signal; however, it tends to over-shrink for higher signal and can have several fold higher error than hard-thresholding.  The EVB approach is a flexible compromise between the two methods, performing comparably or better than both methods for all scenarios and substantially outperforming them for moderate levels of signal.  Moreover, performance of the EVB method is impressively close to the optimal in most scenarios.      

\begin{figure}[!h]
\caption{Error in estimating the underlying low-rank signal for a single matrix under different methods, and under different signal-to-noise (s2n) ratios.  The left-panel gives relative squared error, and the right-panel gives oracle normalized standard error. All axes are on a log-scale.}
\label{fig1}
\includegraphics[width=\textwidth]{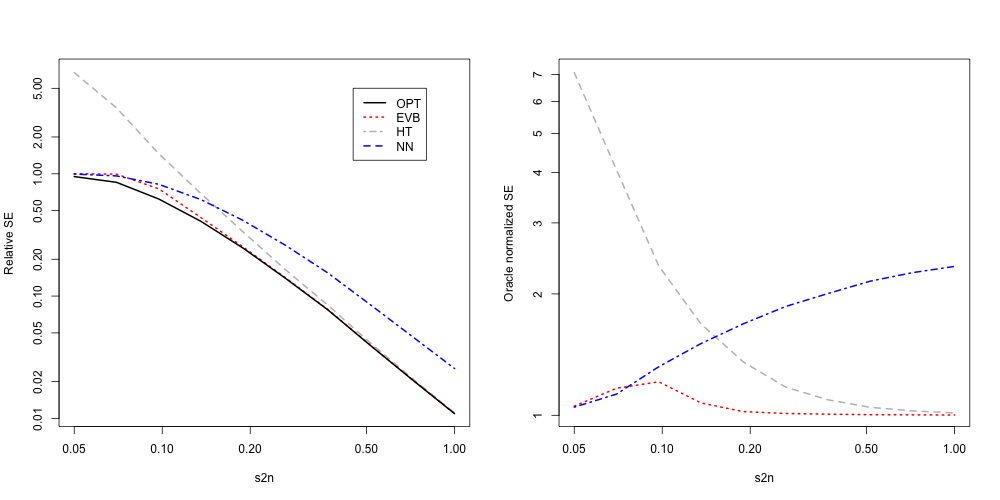}	
\end{figure}

We consider another scenario with low-rank structure in which the underlying components have heterogeneous levels of signal.  First, we simulate $\S'=\U' \V'^T$ where the entries of  $\U: M \times 10$ and $\V: 10 \times N$ are independent standard normal. Then, we generate the signal as \begin{align} \S= \U_{S'} \D_S \V_{S'}^T \label{het.sig} \end{align} where $\D_S$ is diagonal with entries sampled randomly from a log-uniform distribution between $0.05 \sqrt{MN}$ and $\sqrt{MN}$, and $\X=\S+\E$ with entries of $\E$ independent standard normal.  Thus, some singular values of the true low-rank structure $\S$ may correspond to high signal, others may be of moderate signal, and others may be of weak or undetectable signal.  We generate $100$ replications in this case, and for each we apply the EVB, HT, and NN methods described above, as well as the following additional approaches:
\begin{itemize}
  \item HT-OPT: the oracle hard-thresholding approach, in which the rank is selected that results in minimal squared error between the estimated signal and the true signal.  
  \item NN-OPT: the oracle soft-thresholding approach, in which the nuclear norm penalty is selected that results in minimal squared error between the estimated signal and the true signal. 
\item RMT: the low-rank reconstruction approach proposed in \citet{shabalin2013reconstruction}, based on minimizing squared error loss using asymptotic random matrix theory, with the true error variance $\sigma^2=1$. 
\end{itemize}            
Boxplots showing distribution of RSE and ONSE for recovering the underlying structure across replications are shown in Figure~\ref{fig2}.  The EVB and RMT methods have performance close to the oracle, and vastly outperform hard-thresholding and nuclear norm penalized approaches.  The relatively poor performance of HT-OPT suggests that singular values benefit from some shrinkage, while the poor performance of NN-OPT indicates that the level of shrinkage should not be uniform across singular values.  
 \begin{figure}[!h]
\caption{Error in estimating underlying low-rank structure in which the rank-1 components have heterogenous signal sizes.  The left-panel gives relative squared error (RSE), and the right-panel gives oracle normalized standard error (ONSE). All axes are on a log-scale.}
\label{fig2}
\includegraphics[width=\textwidth]{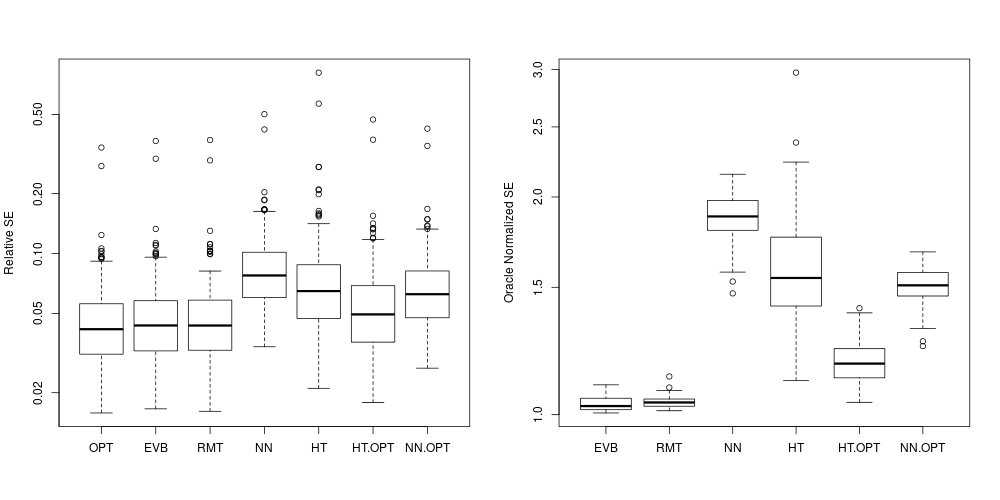}	
\end{figure}

 \begin{figure}[p]
\caption{Missing data imputation accuracy for different levels of missingness. The left column gives $\text{RSE}_{\text{miss}}$ and the right gives $\text{ONSE}_\text{miss}$. }
\label{fig3}
\includegraphics[width=\textwidth]{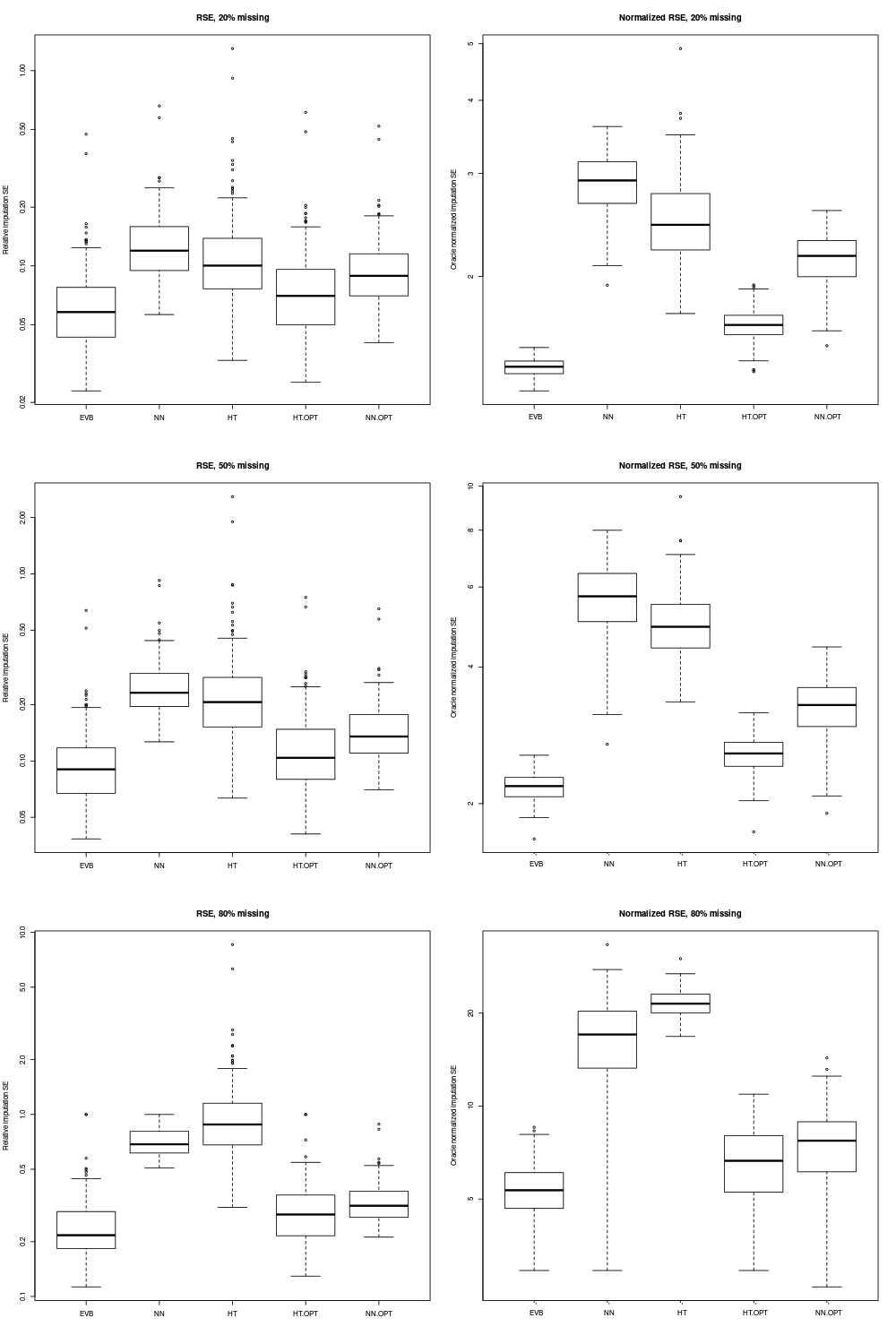}	
\end{figure}

We also compare missing data imputation performance of the EVB approach using Algorithm~\ref{alg2}.  For each of the $100$ simulated datasets generated, we consider scenarios in which $20\%$, $50\%$, or $80\%$ of the entries are randomly held out as missing and imputed.  In addition to EVB imputation, we apply similar EM-imputation methods for the soft-thresholding (corresponding to {\it softImpute} of \citet{mazumder2010spectral}) and hard thresholding \citep{kurucz2007methods} approaches. The RMT approach is not considered here, as it does not have an associated missing data imputation procedure.  For the NN-OPT and HT-OPT approaches, we select the value of $\lambda$ or $R$, respectively, that minimizes squared imputation error for the signal of the missing data 
\begin{align}\text{RSE}_{\text{miss}}=\frac{\sum_{(m,n) \in \mathcal{M}} (\S[m,n]-\hat{\S}[m,n])^2}{\sum_{(m,n) \in \mathcal{M}} \S[m,n]^2}. \label{rsemiss}\end{align}
For oracle normalized imputation error $\text{ONSE}_\text{miss}$, we divide $\text{RSE}_{\text{miss}}$ by the RSE for the OPT method in \eqref{RSE_ONSE}.  The results are shown in Figure~\ref{fig3}, and demonstrate that the EVB imputation approach is stable for different levels of missingness and substantially outperforms the other approaches.

\subsection{Two linked matrices}

Here we generate two matrices, $\X_1: 500 \times 100$ and $\X_2: 500 \times 100$, that are linked by columns: $\X_{\bigcdot}=[\X_1;\X_2]: 1000 \times 100$.  We consider this context primarily to allow comparison of the EV-BIDIFAC approach with several other methods that identify shared and unshared low-rank structure in this setting. 

First, we perform an illustrative simulation akin to that in Figure~\ref{fig1} for a single matrix. We simulate shared structure $\S_0=c \U_0 \V_0^T$ and individual structures  $\S_1=c \U_1 \V_1^T$ $\S_2=c \U_2 \V_2^T$, where $\U_0: 1000 \times 5$, $\V_0: 100 \times 5$, $\U_1: 500 \times 5$, $\U_2: 500 \times 5$, $\V_1: 100 \times 5$, and $\V_2: 100 \times 5$ each have independent entries from a standard normal distribution. The observed data are then generated as $\X_{\bigcdot}=\S_0+[\S_1;\S_2]+\E$ where entries of $\E$ are independent standard normal. Thus $c$ again controls the signal to noise ratio and we consider $10$ values of $c$ distributed on a log-uniform scale between $c=0.05$ and $c=1$, and generate $10$ replications  for each value of $c$.  Here we consider estimates for the following methods:
\begin{itemize}
\item EVB: the EV-BIDIFAC approach described in Algorithm~\ref{alg2}
\item JIVE: the JIVE method implemented in \citet{oconnell2016} for true shared and unshared ranks $R=5, R_1=5, R_2=5$.
\item UNIFAC: the BIDIFAC method \citep{park2020integrative} for vertical integration using the true noise variance $\sigma^2=1$.	
\end{itemize}
Here JIVE is analogous to HT and UNIFAC to NN for a single matrix, as they involve iterative applications HT or NN, respectively, to the shared and unshared components.  For each method we compute RSE for the overall structure $\hat{\S}=\hat{\S}_0+[\hat{\S}_1;\hat{\S}_2]$ as in \eqref{RSE_ONSE}, and also compute the relative decomposition squared error (RDSE) as 
\[\text{RDSE}=\frac{||\S_0-\hat{\S}_0||_F^2+||\S_1-\hat{\S}_1||_F^2+||\S_2-\hat{\S}_2||_F^2}{||\S_0||_F^2+||\S_1||_F^2+||\S_2||_F^2}.\]
The results are shown in Figure~\ref{fig4}. The results for overall RSE are similar to that for a single matrix, with UNIFAC appropriately conservative for lower signal but substantially underperforming JIVE for higher signal due to over-shrinkage.  However, per the RDSE results, UNIFAC can still identify the shared and unshared components more accurately with higher signal -- this phenomenon was noted previously \citep{lock2020bidimensional} and is in part due to orthogonality restrictions imposed by JIVE and related methods that are required for uniqueness of the decomposition without shrinkage.  The EVB approach again acts as a flexible compromise between the two methods, and substantially outperforms both for many s2n scenarios, particularly with respect to RDSE.  

 \begin{figure}[!h]
\caption{Error in estimating the underlying low-rank signal for two linked matrices under different signal-to-noise (s2n) ratios.  The left-panel gives the RSE, and the right-panel gives RDSE. All axes are on a log-scale.}
\label{fig4}
\includegraphics[width=\textwidth]{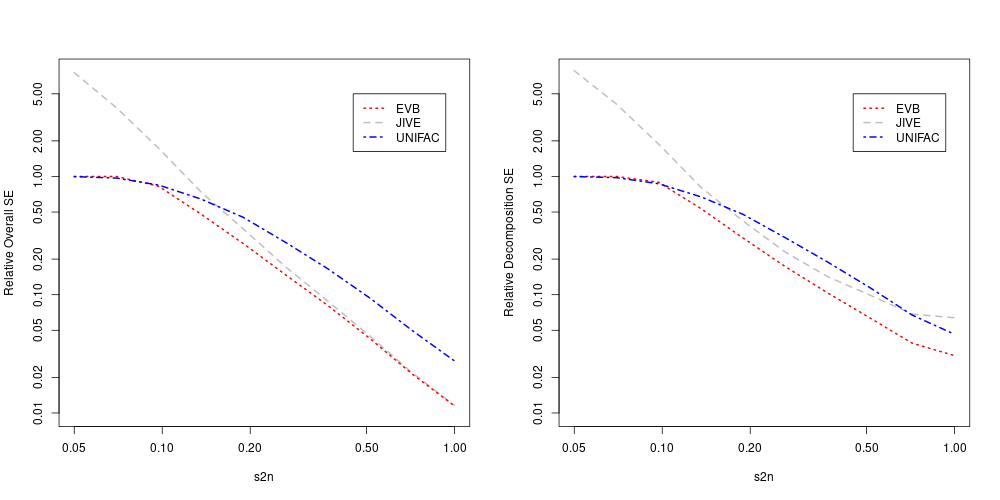}	
\end{figure}

Now, we compare RSE and RDSE under a scenario with heterogenous signal levels for the factors in each of the shared and unshared structures.  The rank-5 signal matrices $\S_0, \S_1,$ and $\S_2$ are each generated as in \eqref{het.sig} with singular values again generated uniformly from a log-uniform distribution between $0.05\sqrt{MN}*0.05$ and $\sqrt{MN}$ for shared structure and between $0.05\sqrt{M_i N}$ and $\sqrt{M_i N}$ for specific structure, and $100$ replications are generated for this scenario.  In addition to the EVB, JIVE, UNIFAC methods we also consider the following approaches:
\begin{itemize}
\item SLIDE: the structural learning and integrative decomposition (SLIDE) method \citep{gaynanova2019structural}, with the true number of shared and unshared components (i.e., the ranks) specified.  
\item HNN.OPT: the hierarchical nuclear norm (HNN) approach \citep{yi2023hierarchical}, with tuning parameters selected on a grid to minimize actual RSE for the true signal.   	
\end{itemize}
 Figure~\ref{fig5} gives distribution of RSE and RDSE over the replications for the different methods. The EVB approach universally outperforms other methods in terms of both RSE and RDSE, and the benefit is large (by a factor of 2 or more) for both criteria in several cases.  While the HNN method is motivated by a model with shared and unshared low-rank structure, it does not provide an explicit decomposition of the structure and thus it is not shown for RDSE.    

 \begin{figure}[!h]
\caption{RSE (left) and RDSE (right) for low-rank structure with heterogenous signal levels for two linked matrices.}
\label{fig5}
\includegraphics[width=\textwidth]{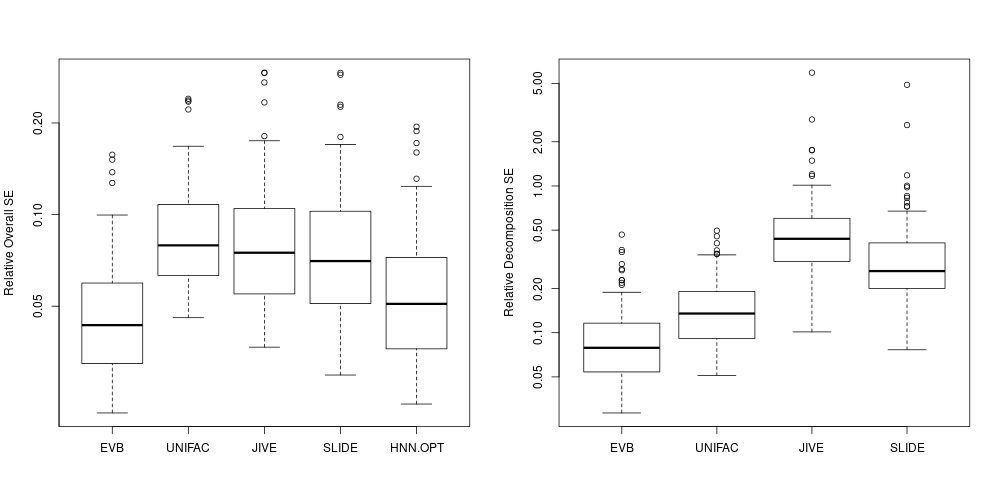}	
\end{figure}

\subsection{Bidimensionally linked matrices}
\label{bidim.sim}

Now, we consider the full bidimensionally linked framework in \eqref{bimat}, with $I=J=2$, $N_1=N_2=50$ and $D_1=D_2=500$.  Data were generated as 
\[\XX=\sum_{k=1}^K \Sk + \EE,\]
where the entries of $\EE$ are independent standard normal.  There are $K=9$ possible distinct modules of low-rank structure in this setting: globally shared, shared across rows for either of the of the $J=2$ column sets, shared across columns for either of the $I=2$ row sets, or structure specific to any of the $4$ matrices $\X_{ij}$.  We generate $100$ datasets, where for each replication $5$ of the $9$ possible modules are randomly selected to have non-zero low-rank structure and the others are set to zero, $\Sk=0$.  All non-zero modules are generated to have rank $2$ structure on the given submatrix as in \eqref{het.sig}.  For each simulated dataset, we apply the following methods:
\begin{itemize}
\item EB-BIDI: the approach in Algorithm~\ref{alg1}, with $\R$ and $\C$ enumerating all $K=9$ possible modules.
\item BIDIFAC: the BIDIFAC algorithm described in~\citep{lock2020bidimensional} with $\R$ and $\C$ enumerating all possible modules and using the true noise variance $\sigma^2=1$.
\item EB-SEP: the EVB approach applied to each matrix $\X_{ij}$ separately 
\item EB-JOINT: the EVB approach applied to the overall matrix $\XX$.	
\end{itemize}
 For each method, we compute the overall RSE for $\S_{\bigcdot \bigcdot}=\sum_{k=1}^K \Sk$ as in \eqref{RSE_ONSE}, and for EB-BIDI and BIDIFAC we compute the RDSE as 
 \[\text{RDSE}=\frac{\sum_{k=1}^9 ||\Sk-\hatSk||_F^2}{\sum_{k=1}^9 ||\Sk||_F^2}.\]
 The results are summarized in Figure~\ref{fig6}.  Overall RSE is substantially better for EB-BIDI vs. other methods; in comparison to EB-SEP and EB-JOINT, this illustrates the advantages of a parsimonious decomposition of the low-rank signal into the different shared and unshared components with respect recovering the overall signals.  Moreover, EB-BIDI has substantially better RDSE compared to BIDIFAC, demonstrating the advantages of the empirical variational Bayes objective with respect to accurately decomposing the structure into its different components.  Moreover, EB-BIDI correctly estimated modules with no true structure ($\Sk=\0$) as $\hatSk=\0$ for all of the $400$ total instances ($100$ total datasets with $4$ possible modules set to $\0$ in each case); on the other hand, it estimated $\hatSk=\0$ in only $19/500=3.8\%$ of instances when $\Sk \neq 0$, which occured when $\Sk$ was difficult to detect due to a small signal size.  This demonstrates how the approach can correctly recover the underlying rank-sparsity structure, only estimating non-zero structure for a module if it truly has distinct low-rank structure on its given rows and columns.               

 \begin{figure}[!h]
\caption{RSE (left) and RDSE (right) for the scenario with $2 \times 2$ bidimensionally linked matrices.}
\label{fig6}
\includegraphics[width=\textwidth]{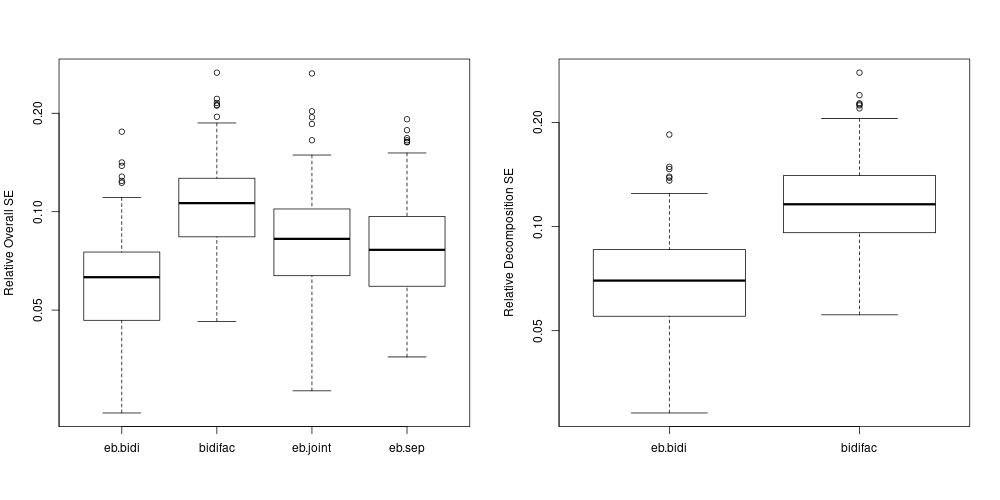}	
\end{figure}

We extend this simulation design to consider missing data imputation accuracy for the generated datasets. We consider two different missingness scenarios: an {\it entry-wise missing} scenario in which 20\% of the entries are randomly set to missing in each matrix, and a {\it blockwise missing} scenario in which 10\% of the columns and 10\% of the rows are randomly selected to have all of their entries missing in each matrix.  In addition to imputation using the previous methods, we also consider the NN.OPT and HT.OPT methods as described in Section~\ref{single_sim}. For entrywise imputation the $\text{RSE}_{\text{miss}}$ is computed as in \eqref{rsemiss}. For blockwise imputation, the true underlying structure is taken to be the sum of column-shared structure for missing rows and the sum of row-shared structure for missing columns, as this is the only estimable signal when an entire row or column is missing.  That is, letting $\S_{\bigcdot \bigcdot}^C=\sum_{k:\C[\cdot,k]=\mathbf{1}} \Sk$ and $\S_{\bigcdot \bigcdot}^R=\sum_{k:\R[\cdot,k]=\mathbf{1}} \Sk$, 
\begin{align}\text{RSE}_{\text{miss}}=\frac{\sum_{(m,n) \in \mathcal{M}^R} (\S^C_{\bigcdot \bigcdot}[m,n]-\hat{\S}_{\bigcdot \bigcdot}[m,n])^2+\sum_{(m,n) \in \mathcal{M}^C} (\S^R_{\bigcdot \bigcdot}[m,n]-\hat{\S}_{\bigcdot \bigcdot}[m,n])^2}{\sum_{(m,n) \in \mathcal{M}^R} \S^C_{\bigcdot \bigcdot}[m,n]^2+\sum_{(m,n) \in \mathcal{M}^C} \S^R_{\bigcdot \bigcdot}[m,n]^2}, \label{rsemiss2}\end{align}
where $\mathcal{M}^R$ indexes row-missing entries and $\mathcal{M}^C$ indexes column missing entries. 
The results are shown in Figure~\ref{fig7}.  The EB-BIDI method substantially and consistently outperforms the other approaches with respect to missing data imputation, and the advantages are particularly large for blockwise missing imputation.  Note that the EB-SEP approach does not estimate any structure for blockwise missing data, as it is not estimable when considering each matrix separately.     
 \begin{figure}[!h]
\caption{RSE for entrywise missing data imputation (left) and blockwise missing data imputation (right), for the scenario with $2 \times 2$ bidimensionally linked matrices.}
\label{fig7}
\includegraphics[width=\textwidth]{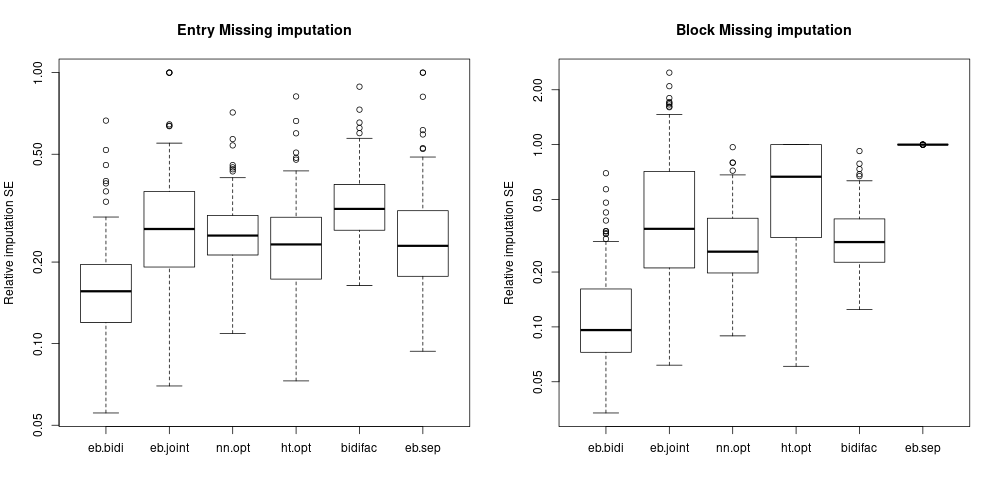}	
\end{figure}

\section{Application to BRCA Data}
\label{app}

Here we describe an application to gene expression (mRNA) and microRNA (miRNA) data for breast cancer (BRCA) from the Cancer Genome Atlas  (TCGA) \citep{cancer2012comprehensive}. The data were processed as described in \citet{park2020integrative}, resulting in $D_1=500$ mRNAs and $D_2=500$ miRNAs for $N_1=660$ cancer tumor samples and $N_2=86$ non-cancerous breast tissue samples from unrelated individuals.  Thus, the data can be represented as bidimensionally linked matrices as in \eqref{bimat}, where $\X_{11}$ is mRNA data for cancer samples, $\X_{12}$ is mRNA data for normal samples, $\X_{21}$ is miRNA data for cancer samples and $\X_{22}$ is miRNA data for normal samples.  The decomposition of data into different modules of shared structured variation is very well-motivated in this context, as low-rank structure on any of the $K=9$ possible modules is plausible and interpretable.  For example, it is reasonable to expect shared structure across mRNA and miRNAs, as miRNAs can regulate mRNA. It is also reasonable to expect some shared structure across tumor and normal samples, as breast tumors arise from normal tissue and so they likely share some patterns of molecular variation.  However, it is also reasonable to expect specific structure in each case, e.g., patterns of molecular variability that are `cancer-specific' (i.e., not present in normal tissue) are of particular interest.  The BIDIFAC approach decomposes such shared and unique structure across both row sets (mRNA/miRNA) and column sets (tumor/normal), while other methods do not. Moreover, as demonstrated in our simulations, the empirical variational (EV) Bayes approach captures underlying low-rank structure more efficiently and accurately than other low-rank approximation methods. Thus, we expect our EV-BIDIFAC approach to decompose the underlying shared and unshared signals in an informative way, and accurately recover the underlying signal in the full dataset $\X_{\cdot \cdot}$ better than alternative methods.

\begin{figure}[!h]
\caption{Heatmaps of the BRCA data (left) and the full low-rank structure estimated by EV-BIDIFAC. Higher values are colored red, lower values are colored blue, and missing columns are colored black.}
\label{fig.heat}
\includegraphics[width=\textwidth]{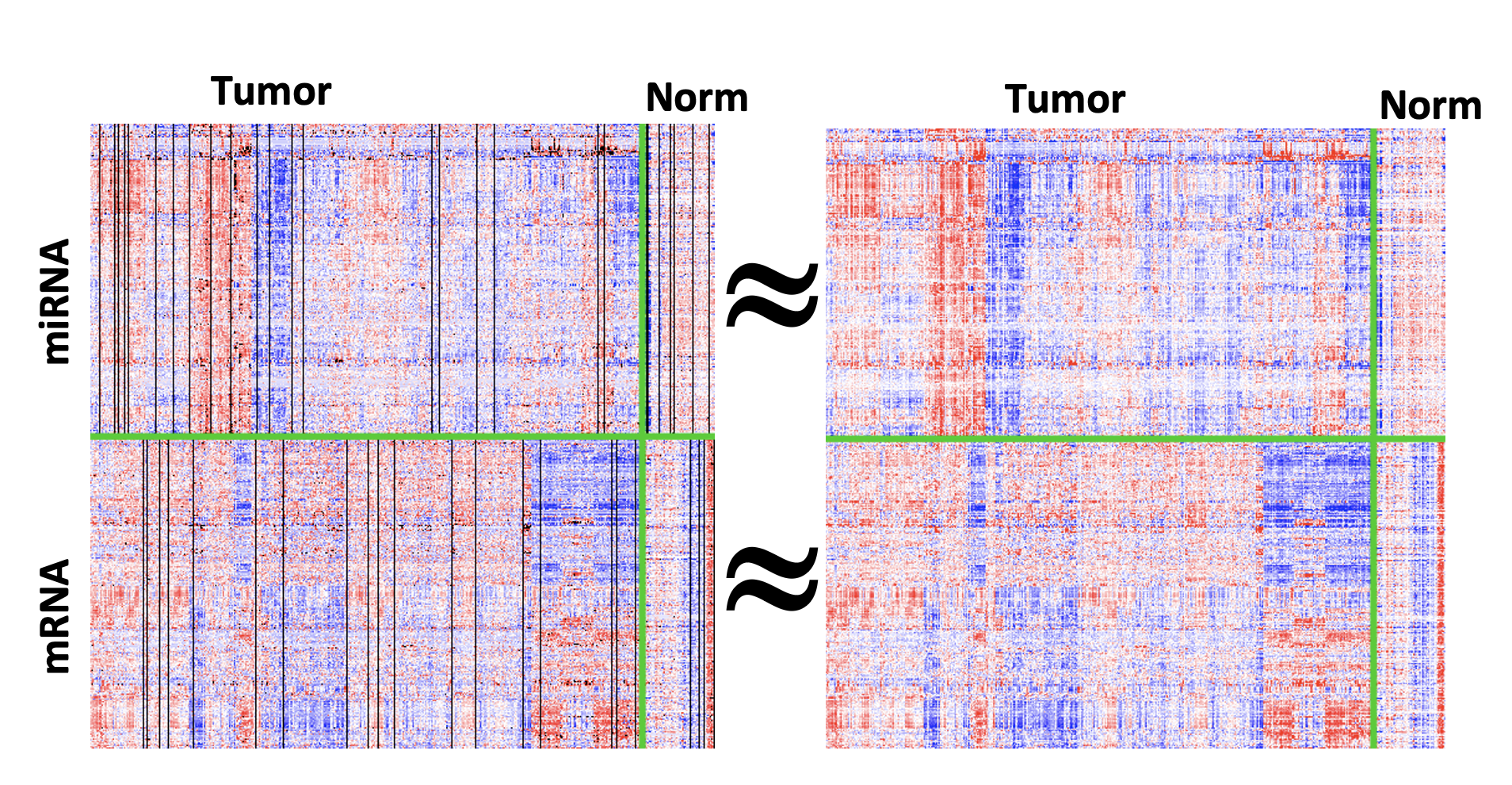}	
\end{figure}

We applied EV-BIDIFAC wherein $\R$ and $\C$ enumerate all $K=9$ possible modules, and significant distinct low-rank structure was identified for each module ($\hatSk \neq \0$ for all $k$).  Figure~\ref{fig.heat} shows heatmaps of the original data (including missing data for miRNA and mRNA) and the resulting overall fit $\hat{\S}_{\bigcdot \bigcdot}=\sum_{k=1}^K \hatSk$.  While all estimated structures are low-rank, it captures the original data well and imputation of the missing columns appears reasonable. Considering the relative amount of variation present in the different modules reveals the extent of shared variation across mRNA and miRNA, and across tumor and normal cells.  Interestingly, only 1.47\% of the low-rank signal is in the module shared globally across all submatrices, while 59.8\% of the signal is shared across mRNA and miRNA but specific to tumors. This suggests that the substantial coordinated activity between mRNA and miRNA in breast tumors is cancer-specific and not present in the normal cells they derive from.   

We further compare imputation accuracy for data under entry-wise, column-wise or row-wise missingness under a cross-validation scheme.  The accuracy of entry-wise imputation is of interest because it is a good proxy for the accuracy of the the overall fit of the estimated structure.  Column-wise imputation accuracy is of interest for two reasons: (1) the actual data have column-wise missingness, as not all samples have both mRNA and miRNA data in the full TCGA cohort, and (2) it provides a robust measure of the level of association between the two data sources, addressing how well can we predict gene expression with miRNA data, and vice versa.  Row-wise imputation accuracy is also of interest, as it provides a measure of similarity in molecular structure between cancer and normal samples (e.g., how well can we predict levels of a missing gene in cancer samples, given its associations in normal samples).

We create 20 folds in the data, where for each fold 5\% of the columns, 5\% of the rows, and 5\% of the remaining entries in each matrix are randomly withheld as missing. For each fold, the missing data are imputed using the EV-BIDIFAC, BIDIFAC, EB-SEP or EB-JOINT approaches described in Section~\ref{bidim.sim}. We further consider nuclear norm penalized imputation (NN) and hard-thresholding (HT) using {\tt softImpute} \citep{softImpute}, jointly for the full matrix (NN-JOINT and HT-JOINT) or separately for each submatrix (NN-SEP and HT-SEP).  In each case the nuclear norm penalty is selected via random matrix theory, and the rank for hard-thresholding is determined via cross-validation by optimizing MSE for an additional 5\% of randomly withheld entries.  As an alternative to low-rank methods we consider k-nearest neighbor (KNN) imputation using the package {\tt impute} \citep{KNN} on the full matrix. Lastly, we consider a structured matrix completion approach \citep{cai2016structured} designed to impute a contiguous missing submatrix under a low-rank approximation using the {\tt StructureMC} package \citep{StructureMC}. As StructureMC imputes only one submatrix, it does not apply to randomly missing entries, and we hold out one set of rows or columns at a time to impute from the full matrix, with all other entries observed.    In each case, the relative squared error for the true observed data is computed and averaged across the four datasets yielding 
\[\text{MRSE}_{\text{miss}}=\frac{1}{4} \sum_{i=1}^2 \sum_{j=1}^2 \frac{\sum_{(m,n) \in \mathcal{M}_{ij}} (\X_{ij}[m,n]-\hat{\S}_{ij}[m,n])^2}{\sum_{(m,n) \in \mathcal{M}_{ij}} \X_{ij}[m,n]^2}.\]
The $\text{MRSE}_{\text{miss}}$ under different imputation scenarios are summarized in Table~\ref{tab1}.  The EV-BIDIFAC approach performs relatively well across all scenarios.  For entry-wise missing data EV-BIDIFAC performs much better than BIDIFAC (likely due to over-shrinkage of the signal for BIDIFAC) and similarly to EB-SEP, indicating that structure can be well-estimated by considering each dataset separately; however, separate estimation does not allow for row- or column imputation and does not estimate the shared variation that is of interest.  The $\text{MRSE}_{\text{miss}}$ for row-missing data is relatively modest ($0.90$) compared to entry-wise missing ($0.38$), indicating that there is little shared information between tumor and normal cells.  The joint approach EB-JOINT on $\XX$ performs relatively poorly across all scenarios as it tends to overfit shared variation when there  substantial specific variation across platforms (mRNA and miRNA) and cohorts (tumor and normal); this illustrates the benefits of a comprehensive decomposition of shared and specific variation across the constituent matrices.      
  \begin{table}[!h]
\caption{Imputation accuracy ($\text{MRSE}_{\text{miss}}$) for BRCA and normal tissue mRNA and miRNA data. The lowest values in each column are shown in \textbf{bold} font; if multiple values are
bolded, then their means were not significantly different at a 0.05 level using a paired t-test over the 20 folds}.  
\label{tab1}
\centering 
\begin{tabular}{c rrrr} 
\hline\hline 
Method Name& Entry-missing & Column-missing & Row-missing & Overall \\ 
\hline 
EV-BIDIFAC & \textbf{0.389} & \textbf{0.756} & \textbf{0.903} & \textbf{0.682}\\ 
BIDIFAC & 0.526 & 0.870 & \textbf{0.909} & 0.768\\
 EB-SEP & \textbf{0.381} & 1.00 & 1.00 & 0.79\\
 EB-JOINT &0.510 & 1.01 & 1.67 & 1.06 \\ 
 NN-SEP & 0.536 & 1.00 & 1.00 & 0.845 \\
NN-JOINT & 0.614 & 0.891 &  \textbf{0.881} & {0.796}  \\
 HT-SEP & 0.459 & 1.00 & 1.03 & 0.831 \\
  HT-JOINT & 0.520 & 4.50 & 13.1 & 6.05 \\
 KNN & 0.646 & 1.26 & 1.01 & 0.974 \\
 StructureMC & - & 0.977 & 1.01 & - \\
\hline 
\end{tabular}
\end{table}

\section{Conclusion and discussion}
\label{discuss}

The proposed EV-BIDIFAC approach can be considered as a flexible compromise between hard- and soft- thresholding techniques for the linked matrix decomposition problem, and has impressive performance with respect to 1.\ estimating underlying low-rank structure, 2.\ decomposing underlying structure into its shared, partially-shared and specific components, and 3.\ missing data imputation under a variety of scenarios.  An appealing aspect of the method is that it derives from a unified and intuitive model with a corresponding objective function that does not require any pre-hoc choices or parameters to tune.  In simulation comparisons the proposed approach tended to outperform other methods even when their parameters are set to the true values (e.g., true rank(s) or true error variance) or ideal values (e.g., to minimize error with knowledge of the true structure).

While we have focused on obtaining a point estimate for underlying structure that minimizes free energy, this could also be used to empirically estimate hyperparameters and initialize a sampling algorithm for the associated fully Bayesian model.  This would capture uncertainty in the underlying factorization, which can be used for multiple imputation of missing values or to propagate uncertainty if the lower-dimensional components are used in subsequent statistical modeling.  A related Bayesian model is described in \citet{samorodnitsky2022bayesian}, which also illustrates the importance of uncertainty propagation. 

Often, data take the form of higher-order tensors instead of matrices, and low-rank tensor factorizations are useful to uncover their underlying structure \citep{kolda2009tensor}. The models considered here may be extended to the tensor setting, e.g., by allowing  factors for each dimension of a tensor to have a distribution analogous to that for $\U$ and $\V$ in the matrix setting. However, the algorithms are not straightforward to extend, as there is likely not a closed-form solution for a tensor akin to that in Theorem~\ref{thm:1} for a matrix.  Extending this approach to the setting of a single tensor or linked tensors is an interesting    direction for future work.

\section*{Supplementary materials} 

Reproducible R scripts for all results presented in this manuscript are available in a supplementary zipped folder.  Software to perform the EV-BIDIFAC method is available at \url{https://github.com/lockEF/bidifac}. Additional simulation study results are available at \url{https://ericfrazerlock.com/ev-bidifac-sims.html}.    

\section*{Acknowledgement}   

This work was supported by the NIH National Institute of General Medical Sciences (NIGMS) grant R01-GM130622. 

\begin{appendices}

\section{Proofs}

\subsection{Proof of Theorem~\ref{thm:ident}}
\label{p1}

The proof for Theorem~\ref{thm:ident} follows a similar structure to that for Theorem 1 in \citet{lock2020bidimensional}. We first establish four lemmas, Lemmas~\ref{lem1}, \ref{lem2}, \ref{lem3}, and \ref{lem4}, which are useful to prove the main result.
 
\begin{lemma} \label{lem1} Take two decompositions  $\SolHat \in \SolSpace$ and $\SolTil \in \SolSpace$, and assume that both minimize the structured penalty for given $\{\blambda\}_{k=1}^K$: 
\[\fpen (\SolHat)=\fpen \left(\SolTil\right)=\underset{\SolSpace}{\min} \; \fpen  (\Sol).\]
Then, for any $\alpha \in [0,1]$,
\begin{align*}
||\alpha \hatSk+(1-\alpha)\tilSk||_* = \alpha ||\hatSk||_*+(1-\alpha) ||\tilSk||_*  
\end{align*}
for $k=1,\hdots,K$.
\end{lemma} 

\begin{proof}
Because $\SolSpace$ is a convex space and $\fpen$ is a convex function, the set of minimizers of $\fpen$  over $\SolSpace$ is also convex.  Thus, 
\[\fpen \left(\{\alpha \hatSk + (1-\alpha) \tilSk\}_{k=1}^K \right)=\underset{\SolSpace}{\min} \; \fpen  (\Sol).\]
The result follows from the convexity of $P_{\blambda}$, which implies that for any two matrices of equal size $\hat{\A}$ and $\tilde{\A}$, 
\begin{align}P_{\blambda}(\alpha \hat{\A} + (1-\alpha) \tilde{\A}) \leq \alpha P_{\blambda}(\hat{\A}) + (1-\alpha)P_{\blambda}(\tilde{\A}). \label{pen_conv} \end{align} 
  Applying~\eqref{pen_conv} to each additive term  in $\fpen$ gives 
\begin{align}
 \fpen \left(\{\alpha \hatSk + (1-\alpha) \tilSk\}_{k=1}^K \right) &\leq \alpha \fpen(\SolHat) + (1-\alpha) \fpen(\SolTil) \label{ineq1} \\
 &=  \underset{\SolSpace}{\min} \; \fpen  (\Sol). \notag 
 \end{align}
Because $\{\alpha \hatSk + (1-\alpha) \tilSk\}_{k=1}^K \in \SolSpace$, the inequality  in~\eqref{ineq1} must be an equality,  and it follows that the inequality~\eqref{pen_conv} must be an equality for each penalized term in the decomposition.   
\end{proof}

\begin{lemma}
\label{lem2}
Take two matrices $\hat{\A}$ and $\tilde{\A}$. If $||\hat{\A}+\tilde{\A}||_*=||\hat{\A}||_*+||\tilde{\A}||_*$, and $\U \D_+ \V^T$ is the SVD of $\hat{\A}+\tilde{\A}$, then $\hat{\A} = \hat{\U} \hat{\D} \hat{\V}^T$ where $\hat{\D}$ is diagonal and $||\hat{\A}||_*=||\hat{\D}||_*$, and $\tilde{\A} = \U \tilde{\D} \V^T$ where $\tilde{\D}$ is diagonal and $||\tilde{\A}||_*=||\tilde{\D}||_*$. 
\end{lemma}
\begin{proof}
This result is proved in Lemma 2 of \citet{lock2020bidimensional}.
\end{proof} 

\begin{lemma}
\label{lem3}
Take two matrices $\hat{\A}$ and $\tilde{\A}$. If $P_{\blambda} (\hat{\A}+\tilde{\A})=P_{\blambda} (\hat{\A}) +P_{\blambda}(\tilde{\A})$, and $\U \D_+ \V^T$ is the SVD of $\hat{\A}+\tilde{\A}$, then $\hat{\A} = \hat{\U} \hat{\D} \hat{\V}^T$ where $\hat{\D}$ is diagonal and $P_{\blambda}(\hat{\A})=P_{\blambda}(\hat{\D})$, and $\tilde{\A} = \U \tilde{\D} \V^T$ where $\tilde{\D}$ is diagonal and $P_{\blambda}(\tilde{\A})=P_{\blambda}(\tilde{\D})$. 
\end{lemma}
\begin{proof}
The result follows from repeated applications of Lemma~\ref{lem2} to the rank-1 terms of $\hat{\A}$ and $\tilde{\A}$.  Note that $P_{\blambda} (\hat{\A}+\tilde{\A})=P_{\blambda} (\hat{\A}) +P_{\blambda}(\tilde{\A})$ implies 
\[P_{\blambda}\left(\sum_r \U_{\hat{\A}}[\cdot,r] \D_{\hat{\A}}[r,r]\V_{\hat{\A}}[\cdot,r]^T+\U_{\tilde{\A}}[\cdot,r] \D_{\tilde{\A}}[r,r]\V_{\tilde{\A}}[\cdot,r]^T  \right) = \sum_r \blambda[r] (\D_{\hat{\A}}[r,r] + \D_{\tilde{\A}}[r,r]),\]
which requires that \begin{align*}|| \U_{\hat{\A}}[\cdot,r] \D_{\hat{\A}}[r,r]\V_{\hat{\A}}[\cdot,r]^T+\U_{\tilde{\A}}[\cdot,r] \D_{\tilde{\A}}[r,r]\V_{\tilde{\A}}[\cdot,r]^T ||_*\\=|| \U_{\hat{\A}}[\cdot,r] \D_{\hat{\A}}[r,r]\V_{\hat{\A}}[\cdot,r]^T||_*+||\U_{\tilde{\A}}[\cdot,r] \D_{\tilde{\A}}[r,r]\V_{\tilde{\A}}[\cdot,r]^T ||_* \end{align*} 
for each $r$.
\end{proof} 

\begin{lemma}
\label{lem4}
Take two decompositions  $\SolHat \in \SolSpace$ and $\SolTil \in \SolSpace$, and assume that both satisfy 
\[\fpen (\SolHat)=\fpen \left(\SolTil\right)=\underset{\SolSpace}{\min} \; \fpen  (\Sol).\] Then, $\hatSk = \U_{\bigcdot}^{(k)} \hat{\D} \V_{\bigcdot}^{(k) T}$ and $\hatSk = \U_{\bigcdot}^{(k)} \tilde{\D}^{(k)} \V_{\bigcdot}^{(k) T}$ where 
$\U_{\bigcdot}^{(k)}: M \times R_k$ and $\V_{\bigcdot}^{(k)}: N \times R_k$ have orthonormal columns, and
$\hat{\D}^{(k)}$ and $\tilde{\D}^{(k)}$ are diagonal.
\end{lemma}
\begin{proof}
The result follows as a direct corollary of Lemmas \ref{lem1} and \ref{lem3},  as Lemma~\ref{lem1} implies $P_{\blambda}(\hatSk+\tilSk)=P_{\blambda}(\hatSk|)+P_{\blambda}(\tilSk)$ for each $k$, and then Lemma~\ref{lem3} 
implies the result. 	
\end{proof}

Theorem~\ref{thm:ident} is then established as follows:\\

\begin{proof}
Take two decomposition $\SolHat$ and $\SolTil$ that satisfy properties 1., 2., and 3. of Theorem 1; we will show that $\SolHat=\SolTil$. For each $k=1,\hdots,K$, write $\hatSk = \U_{\bigcdot}^{(k)} \hat{\D} \V_{\bigcdot}^{(k) T}$ and $\hatSk = \U_{\bigcdot}^{(k)} \tilde{\D}^{(k)} \V_{\bigcdot}^{(k) T}$ as in Lemma~\ref{lem4}.   Then, it suffices to show that $\hat{\D}^{(k)}[r,r]=\tilde{\D}^{(k)}[r,r]$ for all $k,r$.  

First, consider module $k=1$ with $\R[\bigcdot,1]=[1 \;0 \;\cdots\; 0]^T$ and $\C[\bigcdot,1]=[1 \; 0 \; \cdots \;0]^T$.    By way of contradiction, assume $\hat{\D}^{(1)}[1,1]>0$ and $\tilde{\D}^{(1)}[1,1]=0$.  The linear independence of $\{\V_j^{(k)}[\bigcdot,r]: \hat{\D}^{(k)}[r,r]>0\}$ and $\{\V_j^{(k)}[\bigcdot,r]: \tilde{\D}^{(k)}[r,r]>0\}$ implies that
\[\mbox{row}(\X_{\bigcdot \bigcdot}) = \mbox{span}\{\U_{\bigcdot}^{(k)}[\bigcdot,r]: \hat{\D}^{(k)}[r,r]>0\}=\mbox{span}\{\{\U_{\bigcdot}^{(k)}[\bigcdot,r]: \tilde{\D}^{(k)}[r,r]>0\}.\]
Thus, $\U^{(1)}[\bigcdot,1]] \in \mbox{span}\{\{\U_{\bigcdot}^{(k)}[\bigcdot,r]: \tilde{\D}^{(k)}[r,r]>0\}$, and it follows from the orthogonality of $\U^{(1)}[\bigcdot,1]$ and $\{\U^{(1)}[\bigcdot,r], r>1\}$ that \[\U_{\bigcdot}^{(1)}[\bigcdot,1] \in \mbox{span}\{\{\U_{\bigcdot}^{(k)}[\bigcdot,r]: \tilde{\D}^{(k)}[r,r]>0 \text{ and }  k>1\}.\]
Moreover, because $\U_{i}^{(1)}=\0$ for any $i>1$ and $\{\U_i^{(k)}[\bigcdot,r]: \tilde{\D}^{(k)}[r,r]>0\}$ are linearly independent it follows that
   \begin{align} \U_{\bigcdot}^{(1)}[\bigcdot,1] \in \mbox{span}\{\U_{\bigcdot}^{(k)}[\bigcdot,r]: \tilde{\D}^{(k)}[r,r]>0, \; k>1, \;  \text{ and } \R[i,k]=0 \text{ for any } i>1\}.\label{temp1}\end{align}
Note that \eqref{temp1} implies $\U_{1}^{(1)}[\bigcdot,1] \in \mbox{row}(\X_{12} + \cdots + \mbox{row}(\X_{1J})$,
however, this is contradicted by the linear independence of $\U_{1}^{(1)}[\bigcdot,1]$ and $\{\U_i^{(k)}[\bigcdot,r]: \hat{\D}^{(k)}[r,r]>0, k>1\}$.  Thus, we conclude that $\tilde{\D}^{(1)}[1,1]>0$ implies $\tilde{\D}^{(1)}[1,1]>0$.  Analogous arguments show that $\tilde{\D}^{(k)}[r,r]>0$ if and only if $\tilde{\D}^{(k)}[r,r]>0$ for any pair $(r,k)$.  It follows that  $\{\U_i^{(k)}[\bigcdot,r]: \hat{\D}^{(k)}[r,r]>0 \text{ or } \tilde{\D}^{(k)}[r,r]>0\}$ are linearly independent for $i=1,\hdots I$, and $\{\V_j^{(k)}[\bigcdot,r]: \hat{\D}^{(k)}[r,r]>0 \text{ or } \tilde{\D}^{(k)}[r,r]>0\}$ are linearly independent for $j=1,\hdots,J$.  Thus,   
\begin{align*}
\sum_{k=1}^K \U_{\bigcdot}^{(k)} (\hat{\D}^{(k)}-\tilde{\D}^{(k)}) \V_{\bigcdot}^{(k) T} = \sum_{k=1}^K \hatSk - \tilSk = \X_{\bigcdot \bigcdot} - \X_{\bigcdot \bigcdot}=\0 	
\end{align*}
implies that $\hat{\D}^{(k)}[r,r]=\tilde{\D}^{(k)}[r,r]$ for all $k,r$.  

\end{proof}

\subsection{Proof of Theorem~\ref{thm:3}}
\label{p2}
It is clear that the partial `M-steps' in Algorithm~\ref{alg2} each partially optimize the free energy.  It remains to be shown that the missing entries minimize free energy when they are replaced with their conditional expectation under the model. Note that the free energy \eqref{freeenergy_bidi} can be expressed as 
\[E_q \log \left(\prod_{k=1}^K q_{u,k}(\U_k) q_{v,k} (\V_k) \right) - E_q \log\left(p(\{\U\}_{k=1}^K \mid \bTau_\U) p(\{\V\}_{k=1}^K \mid \bTau_\V)\right)-E_q \log \left(p(\{\XX \mid \U, \V, \bSigma ) \right).\]
Consider the last term in the expression, decomposing the missing and non-missing entries of $\XX$:
\begin{align}
E_q \log \left(p(\XX \mid \U, \V, \bSigma ) \right) &= E_q \log \left(p(\{\XX[m,n]: (m,n) \notin \mathcal{M}\} \mid \U, \V, \bSigma ) p(\{\XX[m,n]: (m,n) \in \mathcal{M}\} \mid \U, \V, \bSigma ) \right) \notag \\
   &= E_q \log \left(p(\{\XX[m,n]: (m,n) \notin \mathcal{M}\} \mid \U, \V, \bSigma ) \right) + \label{t2}\\&\; \; \;\; \;\;\;\;\;\;\;\;\;\;\;\; E_q \log \left(p(\{\XX[m,n]: (m,n) \in \mathcal{M}\} \mid \U, \V, \bSigma \right)). \notag
\end{align}
Missing data $\{\XX[m,n]: (m,n) \in \mathcal{M}\}$ are random variables that are independent and normally distributed given $\{\U, \V\}$ as in \eqref{bayes_mod}.  Thus, the last term in \eqref{t2} can be expressed as follows:
\begin{align*}
	E_q \log \left(p(\{\XX[m,n]: (m,n) \in \mathcal{M}\} \mid \U, \V, \bSigma \right)) &= c_1 + \sum_{(m,n) \in \mathcal{M}} \frac{1}{2 \bSigma^2[m,n]}  E_q (\XX[m,n]-\mathbf{S}_{\bigcdot \bigcdot}[m,n])^2 \\	
	& \geq c_1 + \sum_{(m,n) \in \mathcal{M}} \frac{1}{2 \bSigma^2[m,n]}  E_q(\mathbf{S}_{\bigcdot \bigcdot}[m,n])-\mathbf{S}_{\bigcdot \bigcdot}[m,n])^2 .
\end{align*}
Thus, for any given $q$, the free energy is minimized at $\X_{\bigcdot \bigcdot}[m,n]=E_q S_{\bigcdot \bigcdot}[m,n]=\hat{\S}_{\bigcdot \bigcdot}[m,n]$ for $(m,n) \in \mathcal{M}$. 




\end{appendices}
.


\end{document}